\documentclass[reqno]{amsart}

\usepackage[T1, T2A]{fontenc} 
\usepackage[utf8]{inputenc}

\usepackage{amsaddr}

\usepackage{amsthm}
\usepackage{amssymb}
\usepackage{mathtools}
\usepackage{enumitem}
\usepackage{xcolor}
\usepackage{xparse}
\usepackage{soul}

\usepackage{graphicx}
\usepackage[font=small]{caption}
\usepackage[labelformat=empty, position=top]{subcaption}
\usepackage[export]{adjustbox}

\usepackage{booktabs}
\usepackage{makecell}

\usepackage{hyperref}
\hypersetup{
    colorlinks=true,
    urlcolor=magenta,
    citecolor=blue,
    linkcolor=red
}

\usepackage[
    backend=biber,
    bibstyle=ieee,
    citestyle=numeric-comp,
    autocite=plain,
    sorting=none,
    natbib=true,
    url=false,
    doi=false,
    date=year,
]{biblatex}
\newtheorem{theorem}{Theorem}

\newtheorem{corollary}{Corollary}
\newtheorem{lemma}{Lemma}
\newtheorem{definition}{Definition}
\newtheorem{remark}{Remark}

\newcommand{\SU}{{S}}
\newcommand{\AU}{\mathcal{A}}
\newcommand{\DU}{\mathcal{D}}
\newcommand{\CU}{\mathcal{C}}
\newcommand{\PU}{\mathrm{P}}
\newcommand{\RU}{\mathrm{R}}
\newcommand{\EU}{\mathrm{E}}

\NewDocumentCommand{\PR}{o o} {
    \IfNoValueTF{#2} {
        \IfNoValueTF{#1} {
            \mathbf{P}
        } {
            \mathbf{P} \left( #1 \right) 
        }
    } {
        \mathbf{P} \left( \left. #1 \right| #2 \right)
    }
}

\NewDocumentCommand{\ER}{o o} {
    \IfNoValueTF{#2} {
        \IfNoValueTF{#1} {
            \mathbf{E}
        } {
            \mathbf{E} \left( #1 \right)
        }
    } {
        \mathbf{E} \left( \left. #1 \right| #2 \right)
    }
}

\NewDocumentCommand{\IR}{o} {
    \IfNoValueTF{#1} {
        \mathbf{I}
    } {
        \mathbf{I} \left\{ #1 \right\}
    }
}

\NewDocumentCommand{\todo}{o} {
    \IfNoValueTF{#1} {
        { \color{red} TODO }
    } {
        { \color{red} TODO: [#1] }
    }
}

\newcommand{\bin}[2]{{ \binom{ #1 }{ #2 } }}

\DeclarePairedDelimiter\abs{\lvert}{\rvert}

\DeclareMathOperator{\crs}{cr}

\title[
    Statistical Verification of Linear Classifiers
]{
    Statistical Verification of Linear Classifiers
}

\author{Zhiyanov A.P.$^{1}$, Shklyaev A.V.$^{2}$, Galatenko A.V.$^{2}$, Galatenko~V.V.$^{2}$ and Tonevitksy A.G.$^{1,3,4}$}
\address{$^{1}$Faculty of Biology and Biotechnology, HSE University, Moscow, 101000, Russia}
\address{$^{2}$Faculty of Mechanics and Mathematics, Lomonosov MSU, Moscow, 119991, Russia}
\address{$^{3}$Shemyakin-Ovchinnikov Institute of Bioorganic Chemistry, Russian Academy of Sciences, Moscow 117997, Russia}
\address{$^{4}$Art Photonics GmbH, Berlin 12489, Germany}

\email[Z]{azhiyanov@hse.ru}

\begin{document}
\maketitle
\begin{abstract}
    We propose a homogeneity test closely related to the concept of linear separability between two samples.
    Using the test one can answer the question whether a linear classifier is merely ``random'' or effectively captures differences between two classes.
    We focus on establishing upper bounds for the test's \emph{p}-value when applied to two-dimensional samples.
    Specifically, for normally distributed samples we experimentally demonstrate that the upper bound is highly accurate.
    Using this bound, we evaluate classifiers designed to detect ER-positive breast cancer recurrence based on gene pair expression.
    Our findings confirm significance of IGFBP6 and ELOVL5 genes in this process.

\end{abstract}

\section{Introduction}
Classification algorithms are widely used in biomedical applications to identify patterns and biological features that differentiate data across various classes~\cite{Cruz2006}.
For example, they can be applied to determine whether a patient has a particular disease based on gene expression profile (diagnostic test)~\cite{Kourou2015}, or to separate patients into groups of high and low risk (prognostic test)~\cite{Kang2013}.

The construction of a classifier involves two major steps: selecting a classification method and choosing a metric to evaluate the classifier's performance~\cite{Murphy2012}.
In biological contexts, linear methods are particularly important.
These classifiers are interpretable, because they assign each feature (e.g., gene expression) a corresponding significance (weight).
Among classification metrics satisfying various natural conditions~\cite{Gosgens2021}, accuracy is considered the simplest and, at the same time, the most interpretable.
Informally, it measures the proportion of correctly classified samples, providing a straightforward evaluation of classifier performance.

In medical applications, many testing systems use a small number of genes to identify specific pathologies~\cite{Paik2004, Vantveer2002, Parker2009}.
As a result, pairs of genes are often used as a starting point for building classifiers with a limited number of features~\cite{Galatenko2015}.
From bioinformatics perspective, construction of such classifiers begins with analysis of expression matrices that include data for thousands of genes.
Based on this matrix, one can identify one or several gene pairs minimizing the number of classification errors.

However, an important question arises:
does the resulting classifier truly capture the differences between classes, or does the observed (possibly small) number of errors result solely from testing a large number of gene pairs?
We demonstrate that upper bounds on the probability of \emph{near-linear separability}  (i.e., separability with several errors) can be used to address this question in real-world scenarios, even in the presence of the multiple testing challenge.
Furthermore, we connect this probability to the \emph{p}-value of a corresponding homogeneity test, which reduces the question to a statistical hypothesis testing.

The best classifiers are often tested on an independent testing set, if it is available.
In this case, the issue of multiple testing does not arise, but another challenge frequently emerges:
the small size of this testing set.
We show that, in real world scenarios, even small testing sets can be sufficient to confirm statistical significance of classification ``non-randomness''.

Formally, let $Y = ( Y_i, \ i \leq k )$, $k \in \mathbb{N}$, and $Z = ( Z_i, \ i \leq l )$, $l \in \mathbb{N}$, be two independent samples in $\mathbb{R}^d$, $d \in \mathbb{N}$.
We assume that $Y$ and $Z$ consist of independent identically distributed (i.i.d.) random vectors (r.v.) distributed according to $\mathcal{F}_{Y}$ and $\mathcal{F}_{Z}$, respectively.
Denote by $X = \{ X_i, \ i \leq n \}$, $n = k + l$, the union of these samples.
Suppose that $Y$ and $Z$ are \emph{linearly separable}.
In that case, it is natural to assume that $\mathcal{F}_Y \neq \mathcal{F}_Z$.
But what is the probability of the event $\AU_0$, where two samples are linearly separable under the assumption that their distributions are equal?

Linear separability of gene expression datasets has been previously studied in~\cite{Unger2010}.
The authors developed an algorithm that checks this property for any pair of genes in constant time on average.
Also, they showed that the number of gene pairs which expression level linearly separates normal and tumor samples is significantly greater than expected.

In comparison with~\cite{Unger2010}, we extend the problem of linear separability in several directions.
We consider the more general event $\AU_m$, which represents \emph{near-linear separability} with $m \in \mathbb{N}$ errors, and $\AU_{\leq m} = \bigcup_{h \leq m} \AU_h$ which represents the separability with $m$ errors or less.
In the theoretical component, we derive upper bounds for $\PR[ \AU_{\leq m} ]$ in the case $d = 2$ using two approaches.

In the first approach, we analyze $\PR[ \AU_{\leq m} \, ][ X = \SU ]$, where $\SU = \{ x_i \in \mathbb{R}^2, \ i \leq n \}$ represents a set of all observations.
Under the assumption $\mathcal{F}_Y = \mathcal{F}_Z$, all $\bin{ n }{ k }$ possible partitions of $\SU$ into $Y$ and $Z$ are \emph{equiprobable}.
Here, we derive upper bounds on $\PR[ \AU_{\leq m} \, ][ X = \SU ]$ and construct the corresponding conditional test.
This is particularly relevant for challenging settings (e.g., recurrence prediction), where even the best classifiers often fail to achieve perfect separability on non-trivial training and testing sets.
Consequently, having \emph{p}-value estimates for near-linear separability is essential for applying these estimates in practice.

In the second approach, the upper bounds on integrated probability $\PR[ \AU_{\leq m} ]$ are established and the corresponding integral test is built.
While these bounds rely heavily on results from~\cite{Erdos1973, Dey1998, Pach2006, Leroux2023, Barany1994}, we introduce a novel statistical perspective on this problem.
Furthemore, we propose a new upper bound on $\PR[ \AU_{\leq m} ]$ for normally distributed samples, which are of particular interest in medical applications.
This focus is justified by the observation that most genes exhibit normal or $\log$-normal expression level distributions~\cite{Liu2019}.

In the computational component, we develop an enumerative algorithm to compute $\PR[ \AU'_{\leq m} \, ][ X = \SU ]$ and $\PR[ \AU_{\leq m} \, ][ X = \SU ]$ for a given $\SU$.
The event $\AU'_{\leq m}$ differs from $\AU_{\leq m}$ as follows:
while $\AU_{\leq m}$ controls the maximum number of misclassified points of both classes, $\AU'_{\leq m}$ evaluates classifier accuracy, i.e., the total number of classification errors.
Compared to the algorithm presented in~\cite{Unger2010}, our method addresses cases where $m > 0$ and is applicable for multidimensional data, i.e., not only for $d = 2$, but also for $d > 2$.

We experimentally validate the proposed methods on synthetic datasets and show the theoretical bounds to be tight in the normal case.
Further, we apply the test to the classifiers developed in~\cite{Galatenko2015} for detecting breast cancer recurrence for ER-positive patients and confirm the critical role of ELOVL5 and IGFBP6 gene pair in the differentiation.

\section{Results}
\subsection{Conditional bounds}
Let $\SU = \{ x_i, \ i \leq n \}$ be a set of $n > 0$ points in \emph{general position} in $\mathbb{R}^2$.
Consider a random partition of $\SU$ into two disjoint subsets $Y$ and $Z$ of the given sizes $k > 0$ and $l = n - k > 0$, where $k \leq l$.
For a fixed $\SU$, these partitions are \emph{equiprobable} and have probability $1 / \bin{n}{k}$.

\begin{definition}
    The subsets $Y$ and $Z$ are called \emph{near-linearly seaparable} with \emph{$m$ errors} if there exists a half-plane $H^+$ such that $\abs*{ H^+ \cap Y } = k - m_Y$, $\abs*{ H^- \cap Z } = l - m_Z$ and $\max( m_Y, m_Z) = m$, where $H^- = \mathbb{R}^2 \setminus H^+$.
    If $m_Y + m_Z = m$, we call the subsets nearly separable with \emph{accuracy $m$}.
    If $H^+$ is the positive half-plane of some directed line $\ell$, we call $\ell$ a \emph{separating line}.
\label{def:linear-separability}
\end{definition}


Denote by $\AU_m$ the event that $Y$ and $Z$ are near-linearly separable with $m$ errors.
\begin{lemma}
    Denote by $\AU_{\leq m} = \bigcup_{h \leq m} \AU_h$.
    Then, for $k \geq 2 m - 1$,
    \[
        \PR[ \AU_{\leq m} \, ][ X = \SU ] \leq \bin{k}{m} \bin{l}{m} \, \PR[ \AU_0 \, ][ X = \SU ].
    \]
\label{lem:error-bound}
\end{lemma}

Denote by $\AU'_m$ the event that $Y$ and $Z$ are separable with accuracy $m$.
\begin{lemma}
    Let $\AU'_{\leq m} = \bigcup_{h \leq m} \AU'_h$ and $\PU_m^{n, k} = \max_{\SU} \PR[ \AU'_m \, ][ X = \SU ]$.
    Then,
    \[
        \PR[ \AU'_{\leq m} \, ][ X = \SU ] \leq \sum_{h \leq m} \bin{n}{m} \bin{m}{h} \, \PU_0^{n - m, k - h}.
    \]
\label{lem:accur-bound}
\end{lemma}

Further, we will use the notation $\PU_m = \PR[ \AU_{\leq m} \, ][ X = \SU ]$, omitting the dependence on $\SU$.
As seen from Lemmas~\ref{lem:error-bound} and~\ref{lem:accur-bound}, our main goal is to bound $\PU_0$.
\begin{theorem}
    For a given $S$, the probability of linear separability can be bounded as
    \[
        \PU_0 \leq \max\left( 6.3 \, n \sqrt[3]{k}, 103 \, n / 8 \right) \, / \, \bin{n}{k}.
    \]
\label{thm:main-local}
\end{theorem}

We can use another approach to characterize separability.
\begin{definition}
    Let us call $\varphi \in [0, 2 \pi)$ a \emph{near-separating angle} if there exists a separating line $\ell$ such that the angle between $\ell$ and the horizontal axis equals $\varphi$.
\label{def:separating-angle}
\end{definition}

Let $\mu_{\leq m}$ and $\mu'_{\leq m}$ denote the Lebesgue measures of the sets of all near-separating angles with $\leq m$ errors and with accuracy $\leq m$, respectively.
Obviously, both of them are measurable r.v.
\begin{theorem}
    For any $x$ such that $0 < x \leq 2 \pi$,
    \begin{align*}
        &\PR[ \vphantom{ \mu'_{\leq m} } \mu_{\leq m} > x \, ][ X = \SU ] \leq \frac{ 2 \pi }{ x } \, \frac{ \sum_{h \leq m} \bin{k}{h} \bin{l}{h} }{ \bin{n}{k} }, \\
        &\PR[ \mu'_{\leq m} > x \, ][ X = \SU ] \leq \frac{ 2 \pi }{ x } \, \frac{ \sum_{h \leq m} \bin{k}{h} \bin{l}{m - h} }{ \bin{n}{k} } = \frac{ 2 \pi }{ x } \, \frac{ \bin{n}{m} }{ \bin{n}{k} }.
    \end{align*}
\label{thm:angle}
\end{theorem}

\subsection{Integral bounds}
Now, let the positions $\SU$ of the sample $X$ be random.
Using Lemmas~\ref{lem:error-bound},~\ref{lem:accur-bound} and the following theorem, we bound the probabilities $\PR[ \AU_{\leq m} ]$ and $\PR[ \AU'_{\leq m} ]$ under assumptions about the distribution of $X$.

\begin{theorem}
\begin{enumerate}[wide, labelwidth=!, labelindent=0pt, label=C\arabic*), ref=C\arabic*]
    \item \label{cond:proj-cond-prelim}
    Suppose that $\{ X_i, \ i \leq n \}$ are i.i.d. random planar points ($d = 2$) with a \emph{projectively continuous} distribution, i.e., any projection of $X$ has a continuous distribution.
    Then,
    \begin{equation}
        \PR[ \AU_0 ] = k l \, \ER[ \RU^{k - 1} (1 - \RU)^{l - 1} ],
        \label{eq:main-integral-general}
    \end{equation}
    where $\RU$ is the probability that $X_1$ lies above the line passing through $X_2$ and $X_3$.

    \item \label{cond:proj-cond}
    Under~\ref{cond:proj-cond-prelim}, we have 
    \begin{equation}
        \PR[ \AU_0 ] \leq 10 \, n \sqrt[4]{k} \, / \, \bin{n}{k}.
        \label{eq:main-integral-expected}
    \end{equation}

    \item \label{cond:spher-sym}
    Additionally, suppose that $\{ X_i, \ i \leq n \}$ have a \emph{spherically symmetric} distribution.
    Then,
    \begin{equation}
        \PR[ \AU_0 ] \leq \frac{ 8 \, n}{ \pi } \, / \, \bin{n}{k}.
        \label{eq:main-integral-spherical}
    \end{equation}
    
    \item \label{cond:normal}
    Finally, let us assume $\{ X_i, i \leq n \}$ are \emph{normally distributed} as $\mathcal{N}(\mu, \Sigma)$, where $\mu \in \mathbb{R}^2$ and $\Sigma$ is a non-singular covariance matrix.
    Then,
    \begin{equation}
        \PR[ \AU_0 ] \leq \sqrt{2} \, n \, / \, \bin{n}{k}.
        \label{eq:main-integral-normal}
    \end{equation}
\end{enumerate}
\label{thm:main-integral}
\end{theorem}


\subsection{Homogeneity test}
Using Theorems~\ref{thm:main-local},~\ref{thm:angle} and~\ref{thm:main-integral}, we construct two homogeneity tests based on near-linear separability.

Let $X = \{ X_i, \ i \leq n \}$ be the union of two two-dimensional samples $Y = ( Y_i, \ i \leq k )$ and $Z = ( Z_i, \ i \leq l )$, where $n = k + l$.
Suppose the elements of $Y$ and $Z$ to be identically distributed (generally, we do not assume them to be independent) with distributions $\mathcal{F}_Y$ and $\mathcal{F}_Z$, respectively.
Consider the hypothesis $\mathrm{H}_0$: $\mathcal{F}_Y = \mathcal{F}_Z$ that these distributions are equal.
Denote by $\mathcal{F}$ the resulting distribution.

We assume $\mathcal{F}$ to be projectively continuous.
Additionally, we consider two kinds of dependency assumptions:
\begin{enumerate}[label=A\arabic*), ref=A\arabic*]
    \item \label{cond:exchangeability}
    under $\mathrm{H}_0$, $\{ X_i, \ i \leq n \}$ are exchangeable, i.e., for any permutation $\sigma$,
    \[
        \left( X_{\sigma(1)}, \dots, X_{\sigma(n)} \right) \stackrel{d}{=} ( X_1, \dots, X_n );
    \]
    \item \label{cond:independence}
    under $\mathrm{H}_0$, $\{ X_i, \ i \leq n \}$ are i.i.d.
\end{enumerate}

\begin{remark}
    Exchangeability is a weaker assumption than independence.
    For instanse, if outliers are removed using a procedure based on $X$ (i.e., without considering labels $Y$ and $Z$), the modified sample $X'$ consists of exchangeable but not i.i.d. random vectors. 
\end{remark}

The first test is based on near-linear seaparability of $Y$ and $Z$.
It rejects $\mathrm{H}_0$ if the number of errors $m$ is small enough.
\begin{corollary}[Linear test]
    Assume that the samples $Y$ and $Z$ are near-linearly separable with $m$ errors or less, i.e., $\AU_{\leq m}$ holds.
    Then, the \emph{p}-value for rejecting $\mathrm{H}_0$ can be upper-bounded by
    \begin{equation}
        \max\left( 6.3 \, n \sqrt[3]{k}, 103 \, n / 8 \right) \, \frac{ \bin{k}{m} \bin{l}{m} }{ \bin{n}{k} },
    \label{eq:exch-test}
    \end{equation}
    if~\ref{cond:exchangeability} holds and $\mathcal{F}$ is projectively continuous;
    \begin{equation}
        10 \, n \sqrt[4]{k} \, \frac{ \bin{k}{m} \bin{l}{m} }{ \bin{n}{k} },
    \label{eq:indep-test}
    \end{equation}
    if~\ref{cond:independence} holds and $\mathcal{F}$ is projectively continuous;
    \begin{equation}
        \frac{8 \, n}{\pi} \, \frac{ \bin{k}{m} \bin{l}{m} }{ \bin{n}{k} },
    \label{eq:spherical-test}
    \end{equation}
    if~\ref{cond:independence} holds and $\mathcal{F}$ is spherically symmetric;
    \begin{equation}
        \sqrt{2} \, n \, \frac{ \bin{k}{m} \bin{l}{m} }{ \bin{n}{k} },
    \label{eq:normal-test}
    \end{equation}
    if~\ref{cond:independence} holds and $\mathcal{F}$ is normal.
\label{cor:linear-test}
\end{corollary}

\begin{remark}
    Using Lemma~\ref{lem:accur-bound}, a similar \emph{p}-value bound can be derived for the test based on $\AU'_{\leq m}$ instead of $\AU_{\leq m}$.
\end{remark}

The second test uses the statistics $\mu_{\leq m}$ and $\mu'_{\leq m}$, which are the Lebesgue measures of the sets of all near-separating angles.
Here, $\mathrm{H}_0$ is rejected if the measure is sufficiently large.
\begin{corollary}[Angle test]
    Assume that~\ref{cond:exchangeability} holds under $\mathrm{H}_0$.
    Then, the \emph{p}-value of the test can be upper-bounded using Theorem~\ref{thm:angle} with $x$ either equal $\mu_{\leq m}$ or $\mu'_{\leq m}$.
\label{cor:angle-test}
\end{corollary}

\subsection{Permutational test}
In addition to the presented upper bounds, the \emph{p}-value of Linear test can be estimated numerically for a fixed $X = \SU$.
Specifically, the \emph{p}-value can be computed as a fraction of partions of $X$ into $Y$ and $Z$ that satisfy $\AU_{\leq m}$ ($\AU'_{\leq m}$).
However, the computational complexity of a naive implementation of this procedure grows exponentially as $n \to \infty$ for $k = l = n / 2$, since there are $\bin{n}{k}$ possible partitions.

We use another approach.
For a fixed number of errors $m$, the algorithm iterates over all possible hyperplanes passing through $d$ points of the initial point set $\SU = \{ x_i, \ i \leq n \}$, where $x_i \in \mathbb{R}^d$, $d \geq 1$, which takes $O\left( n^d \right)$ operations.
Then, it distributes these points into the upper and lower half-spaces ($O(2^m)$ operations) and checks whether the half-spaces contain $k$ and $l$ points, accounting for $m$ possible errors ($O(n)$ operations).
The $2^m$ factor arises from enumerating all possible subsets of size $m$ that are classification errors.

As seen, the overall complexity of the algorithm grows exponentially as $m \to \infty$.
To address this issue, approximate computation of the \emph{p}-value can be employed.
Indeed, for a given $\varepsilon > 0$, due to Central Limit Theorem, we can perform $O\left( \varepsilon^{-2} \right)$ permutations of $Y$ and $Z$ labels to estimate the \emph{p}-value with an accuracy of $\varepsilon$.
The complexity of the described procedure is $O\left( \varepsilon^{-2} \, n^{d + 1} \right)$.

An alternative approximation eliminates the need for exhaustive enumeration of all separating hyperplanes.
Using the Support Vector Machine (SVM) classification procedure~\cite{Vapnik1997} with a linear kernel and a large regularization parameter $C$, the complexity can be reduced to $O(\varepsilon^{-2} \, n^3)$~\cite{Bottou2019}.
This method provides an efficient and scalable way to estimate the \emph{p}-value for big datasets.

All algorithms were implemented using a parallel paradigm and are publicly available at~\url{https://github.com/zhiyanov/random-classifier}.

\subsection{Synthetic experiments}
Using the permutational approach to compute \emph{p}-value, we evaluated the accuracy of the theoretical bound in the normal case.
We considered two samples, $Y$ and $Z$, of equal sizes $k = l = n / 2$, drawn from either the same normal $\mathcal{N}(0, E)$ or uniform $\mathcal{U}[0, 1]^2$ distribution.
The probability estimation of linear separability was computed, with each experiment repeated 10 times for each $n$.

\begin{figure}[h!]
    \centering
    \begin{subfigure}[t]{0.014\textwidth}
        \textbf{A}
    \end{subfigure}
    \begin{subfigure}[t]{0.470\textwidth}
        \centering
        \includegraphics[width=\textwidth, valign=t]{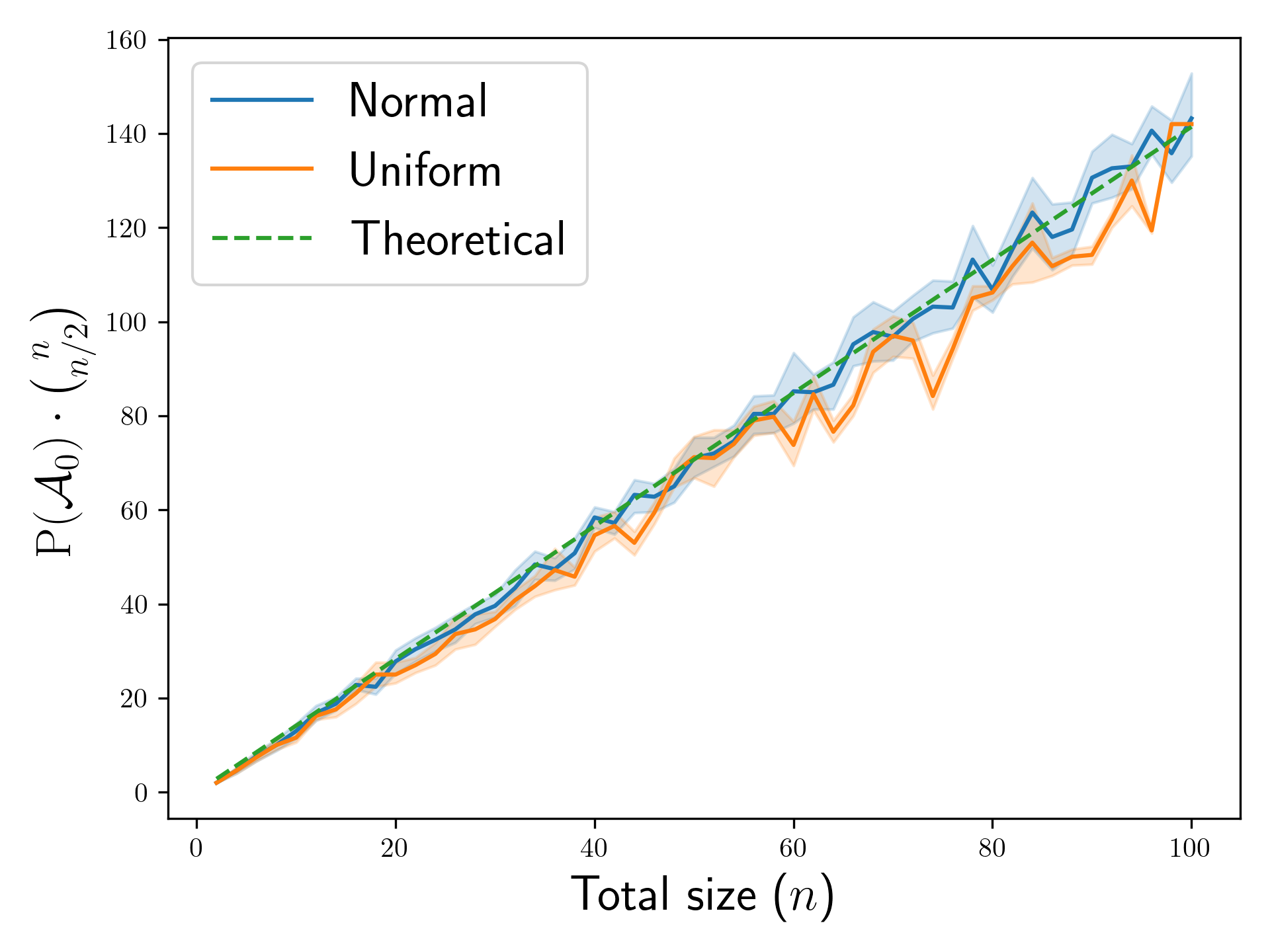}
    \end{subfigure}
    \hfill
    \begin{subfigure}[t]{0.014\textwidth}
        \textbf{B}
    \end{subfigure}
    \begin{subfigure}[t]{0.470\textwidth}
        \centering
        \includegraphics[width=\textwidth, valign=t]{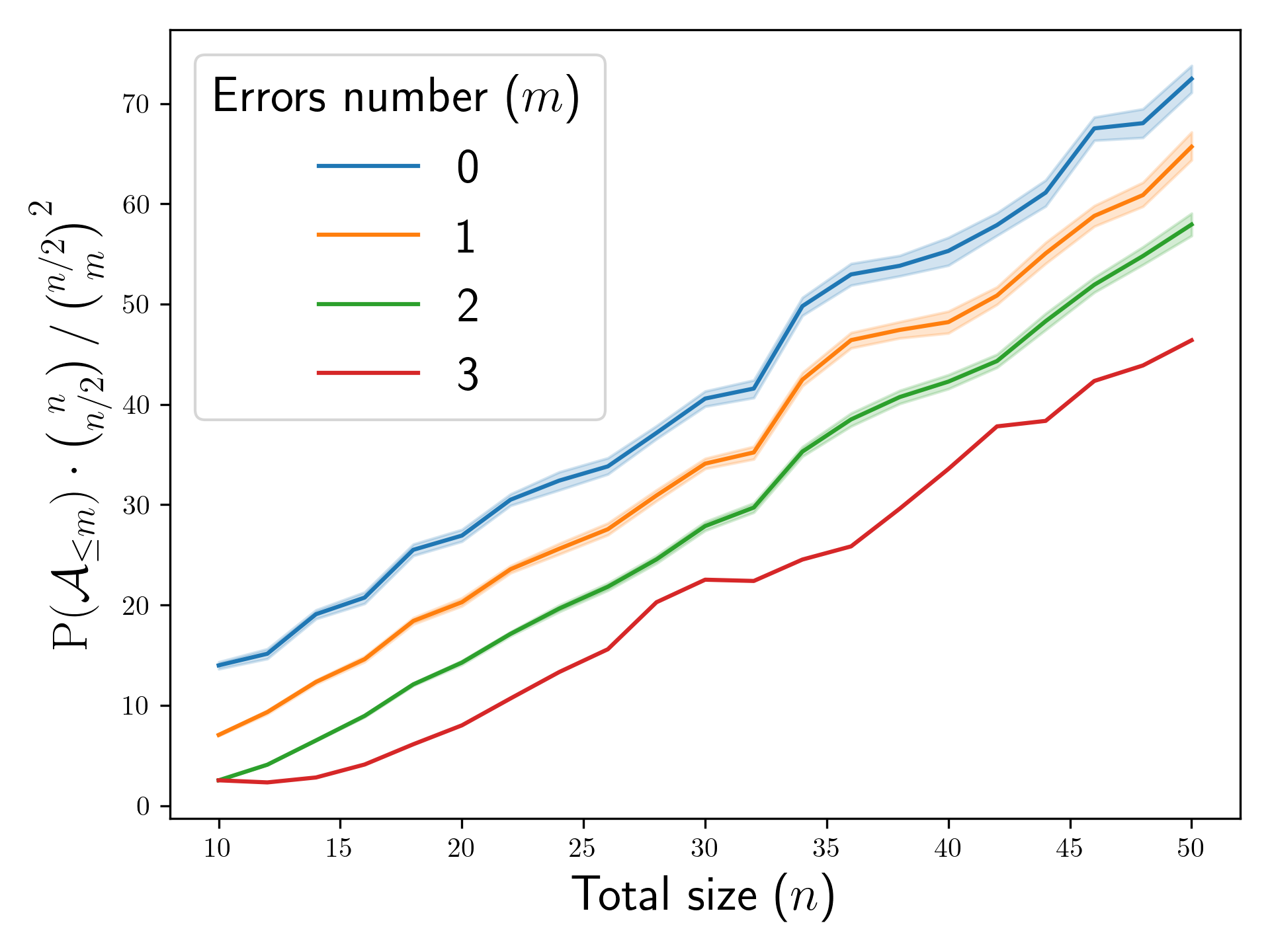}
    \end{subfigure}
    \caption{
        \textbf{A}: The upper bound~\eqref{eq:main-integral-normal} of $\PR[ \AU_0 ]$ in the normal case is tight.
        \textbf{B}: The upper bound from Lemma~\ref{lem:error-bound} approximates $\PR[ \AU_{\leq m} ]$ reasonably well.
    }
    \label{fig:bounds}
\end{figure}

As shown in Figure~\ref{fig:bounds} (A), the upper bound~\eqref{eq:main-integral-normal} was accurate for normally distributed sample.
Also, the probability of linear separability was lower for the uniform samples compared to the normal case.
This observation aligns with the estimates for the uniform samples proposed in~\cite{Barany1994}.

For the normal case, we also examined how the probability of linear seaparability increased as the number of errors $m$ grew.
Figure~\ref{fig:bounds} (B) illustrates that the right side of the inequality of Lemma~\ref{lem:error-bound} was approximately $2^m$ times greater than $\PR[ \AU_0 ]$.

To assess the statistical power of Permutational test, we conducted the following procedure.
For $k = l = 30$, we sampled $Y$ from $\mathcal{N}(0, E)$ (or $\mathcal{U}[0, 1]^2$) and $Z$ from $\mathcal{N}((\mu, 0), E)$ (or $\mathcal{U}[\mu, \mu + 1] \times [0, 1]$), where $\mu \geq 0$.
The \emph{p}-value was estimated with an accuracy of $5 \cdot 10^{-2}$ for each configuration.
The statistical power was then estimated at a significance level of $\alpha = 0.1$ based on 100 repetitions.

As shown in Figure~\ref{fig:power} (B), the power began to increase at $\mu = 0.6$ and $\mu = 0.2$ for the normal and uniform samles, respectively.
This difference can be partially attributed to the fact that standard deviations of the normal and uniform distributions differed by the factor of $\sqrt{12}$.

\begin{figure}[h!]
    \centering
    \begin{subfigure}[t]{0.014\textwidth}
        \textbf{A}
    \end{subfigure}
    \begin{subfigure}[t]{0.470\textwidth}
        \centering
        \includegraphics[width=\textwidth, valign=t]{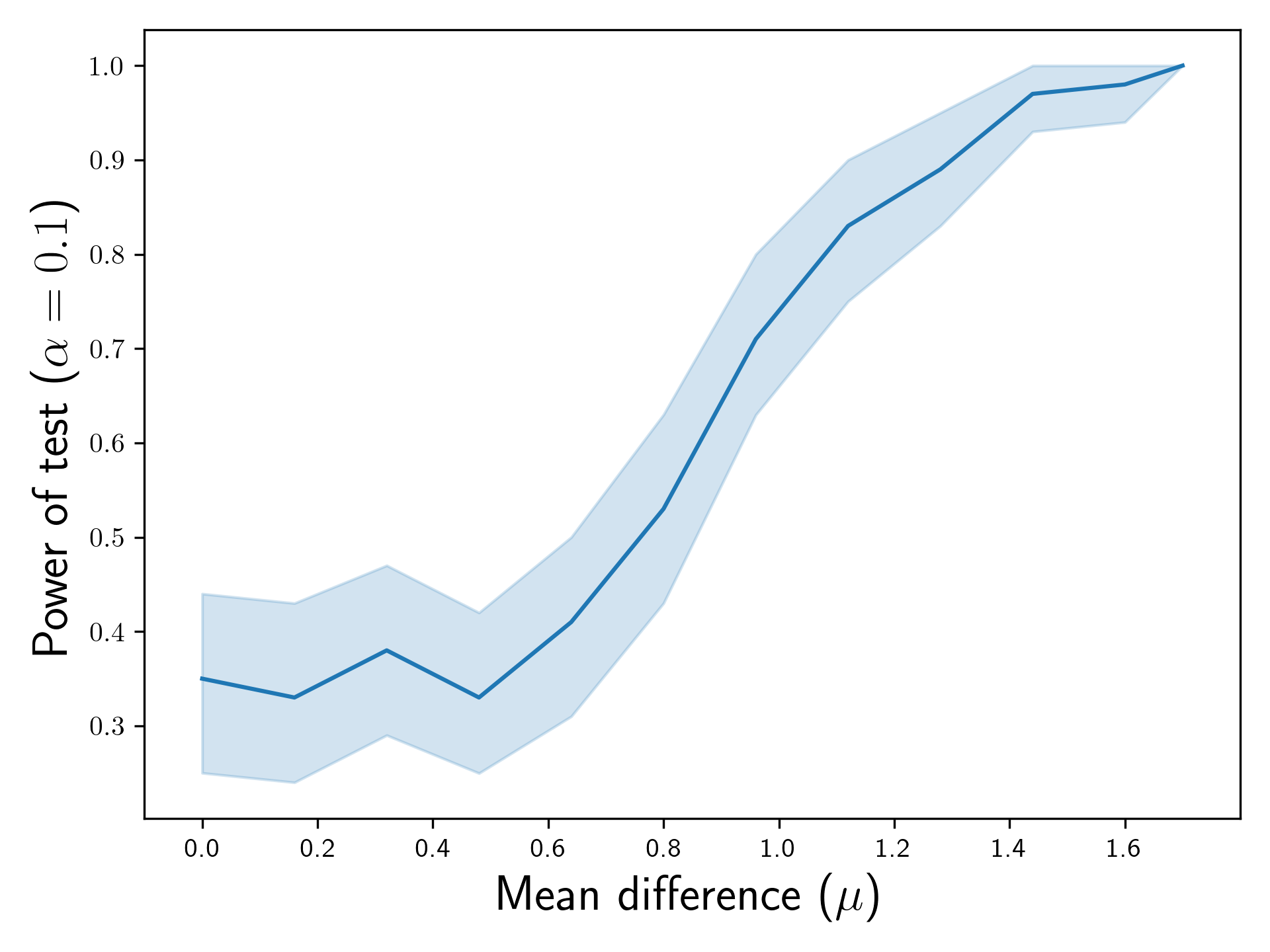}
    \end{subfigure}
    \hfill
    \begin{subfigure}[t]{0.014\textwidth}
        \textbf{B}
    \end{subfigure}
    \begin{subfigure}[t]{0.470\textwidth}
        \centering
        \includegraphics[width=\textwidth, valign=t]{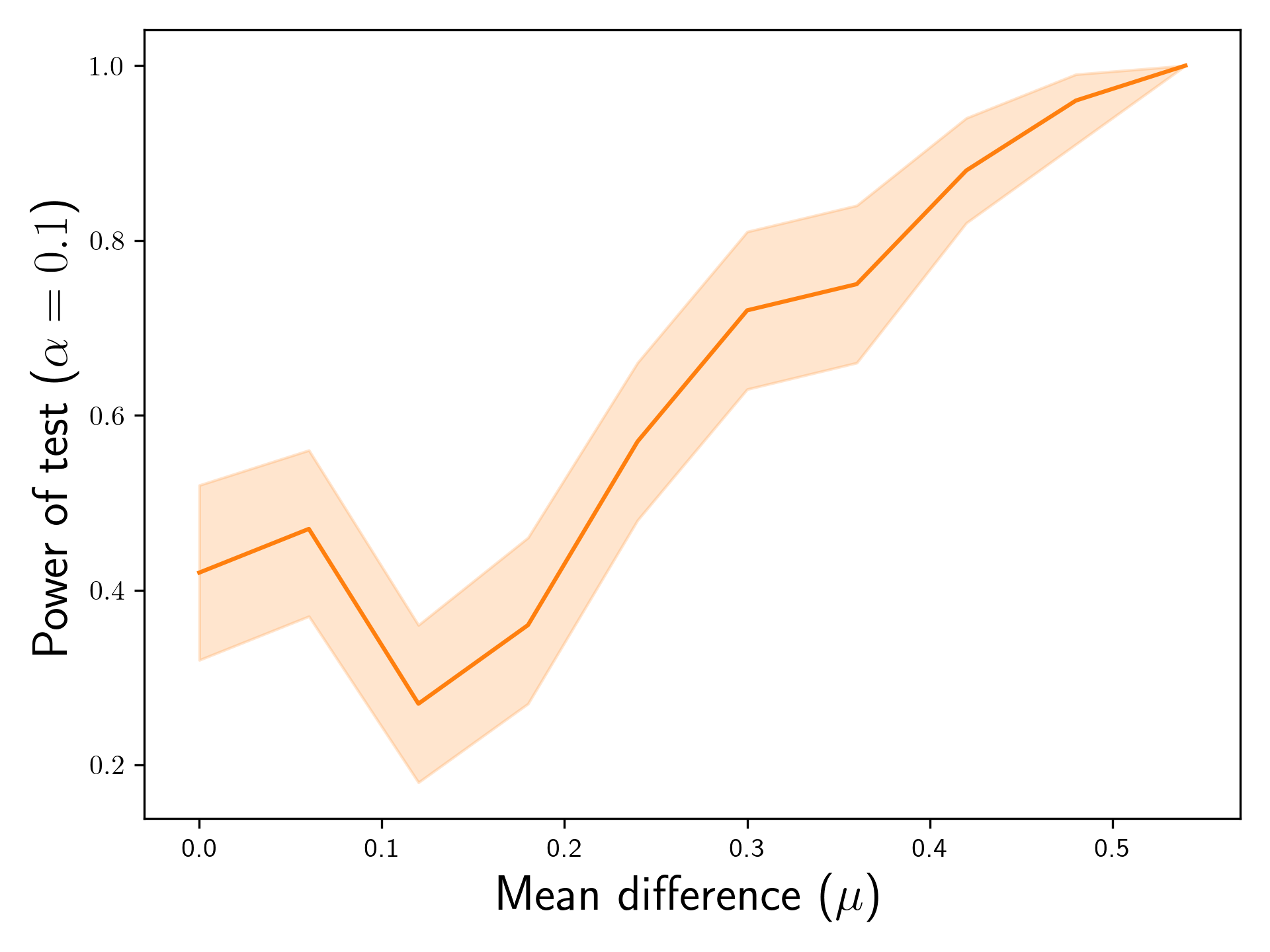}
    \end{subfigure}
    \caption{
        The power of Permutational test
        \textbf{A}: for the normal samples $\mathcal{N}(0, E)$ and $\mathcal{N}((0, \mu), E)$,
        \textbf{B}: for the uniform samples $\mathcal{U}[0, 1]^2$ and $\mathcal{U}[\mu, \mu + 1] \times [0, 1]$,
        where $\mu \geq 0$.
    }
    \label{fig:power}
\end{figure}

\subsection{Application to breast cancer recurrence data}
Previously, the authors of~\cite{Galatenko2015} built linear classifiers detecting risk of ER-positive breast cancer recurrence by an expression level of a gene pair.
Out of all possible pairs of expressed genes, 570 gene pairs with high classification quality on 4 independent microarray datasets were chosen as potential biomarkers of the recurrence.

Our first objective was to determine whether the selected pairs genuinely captured meaningful expression patterns or whether their high classification performance was merely a result of multiple testing.
The corresponding classification metrics, specifically the number of errors $m$, were also obtained from~\cite{Galatenko2015}.

\begin{table}[h!]
\begin{center}
\begin{tabular}{lcccc}
    \toprule
    Dataset & \makecell{Recurrence \\ within 5 years} & \makecell{No recurrence \\ within 7 years} & \makecell{Total \\ size} & \makecell{Number of errors} \\
    \midrule
    Training & 42 & 159 & 201 & 42 $\pm$ 7 \\
    Filtration 1 & 49 & 76 & 125 & 25 $\pm$ 5 \\
    Filtration 2 & 12 & 67 & 79 & 26 $\pm$ 7 \\
    Validation & 33 & 128 & 161 & 43 $\pm$ 15 \\
    \bottomrule
\end{tabular}
\end{center}
\caption{
    Characteristics of microarray datastets used to train and test classifiers detecting ER-positive breast cancer recurrence.
    The column ``Number of errors'' is presented in the form $\mu \pm \sigma$, where $\mu$ is the average of the number of errors $m$ and $\sigma$ is the standard deviation. 
}
\label{tab:micro-meta}
\end{table}

Assuming that expression level for the majority of genes was distributed ($\log$)nor\-mally~\cite{Liu2019}, we applied the bound~\eqref{eq:normal-test} to estimate the ``randomness'' of the obtained classifiers.
As seen from Table~\ref{tab:micro-meta}, Filtration 2 and Validation datasets exhibited the highest error rates relative to their sample sizes, thereby increasing the binomial factor in~\eqref{eq:normal-test}.
Simultaneously, these datasets also had the greatest imbalance between the classes, which reduced this factor.
As a result, Filtration 2 dataset produced the smallest number of potentially ``random'' classifiers (i.e., classifiers with a false discovery rate (FDR) exceeding 0.05), amounting to 13 gene pairs.
In comparison, Validation and Training datasets resulted in 180 and 267 pairs, respectively, while for Filtration 1 dataset, almost all classifiers (559 out of 570) were estimated to be ``random''.
The exact values of the upper bounds on FDRs and \emph{p}-values for all classifiers are provided in Supplementary Table.

The gene pair IGFBP6 and ELOVL5 achieved the highest AUC score among the selected pairs on Validation dataset, making this pair a strong candidate for further evaluation.
The corresponding classifier was tested on an independent RNA-Seq dataset from The Cancer Genome Atlas (TCGA, \url{http://cancergenome.nih.gov/}).
This dataset included 15 recurrent and 17 non-recurrent samples (Figure~\ref{fig:igfbp6-elovl5}), posing the additional challenge of small sample size for statistical inference.

Using the theoretical bound~\eqref{eq:indep-test}, we estimated the \emph{p}-value for the test based on $\AU_{\leq 3}$ to be upper-bounded by 0.35.
Consequently, at a significance level of 0.05, we could not reject the hypothesis $\mathrm{H}_0$.
However, when the \emph{p}-value was computed directly as $\PR[ \AU_{\leq 3} \, ][ X = \SU ]$ or estimated using the bound~\eqref{eq:normal-test}, the result was approximately 0.013 or 0.025, respectively, allowing rejection of $\mathrm{H}_0$ at the 0.05 level.
For the accuracy-based test using $\AU'_{\leq 4}$, the corresponding theoretical estimates upper bounded a \emph{p}-value by 8.86 and 0.67, while the computational estimate of $\PR[ \AU'_{\leq 4} \, ][ X = \SU ]$ yielded the \emph{p}-value of approximately 0.002, further supporting the rejection of $\mathrm{H}_0$.

By comparison, when the labels were randomly permuted, the best SVM classifier produced 10 errors.
The estimated \emph{p}-value for this classifier based on accuracy was 0.76, significantly exceeding the 0.05 threshold and indicating no evidence against ``randomness''.

\begin{figure}[h!]
    \centering
    \begin{subfigure}[t]{0.014\textwidth}
        \textbf{A}
    \end{subfigure}
    \begin{subfigure}[t]{0.470\textwidth}
        \centering
        \includegraphics[width=\textwidth, valign=t]{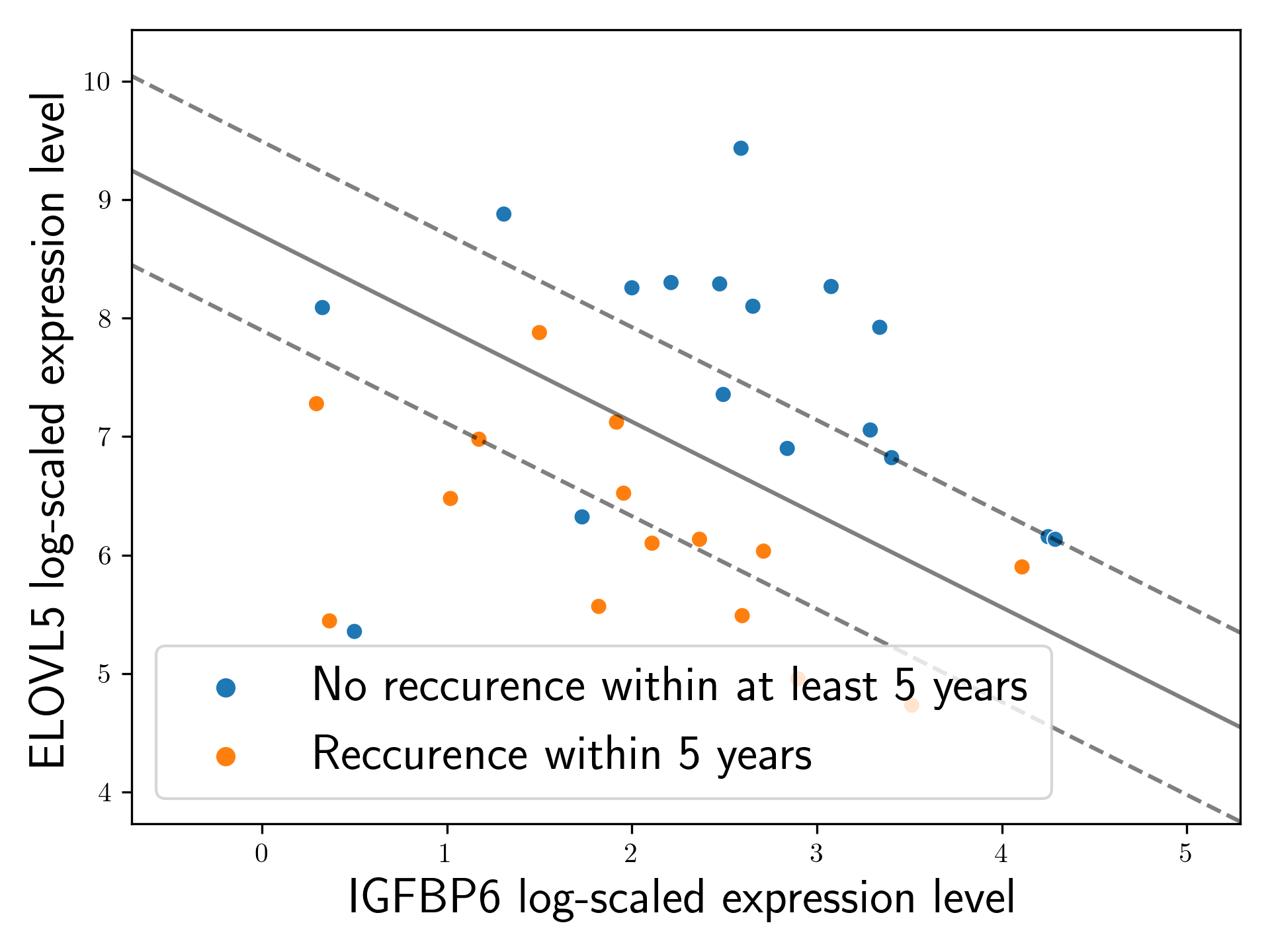}
    \end{subfigure}
    \hfill
    \begin{subfigure}[t]{0.014\textwidth}
        \textbf{B}
    \end{subfigure}
    \begin{subfigure}[t]{0.470\textwidth}
        \centering
        \includegraphics[width=\textwidth, valign=t]{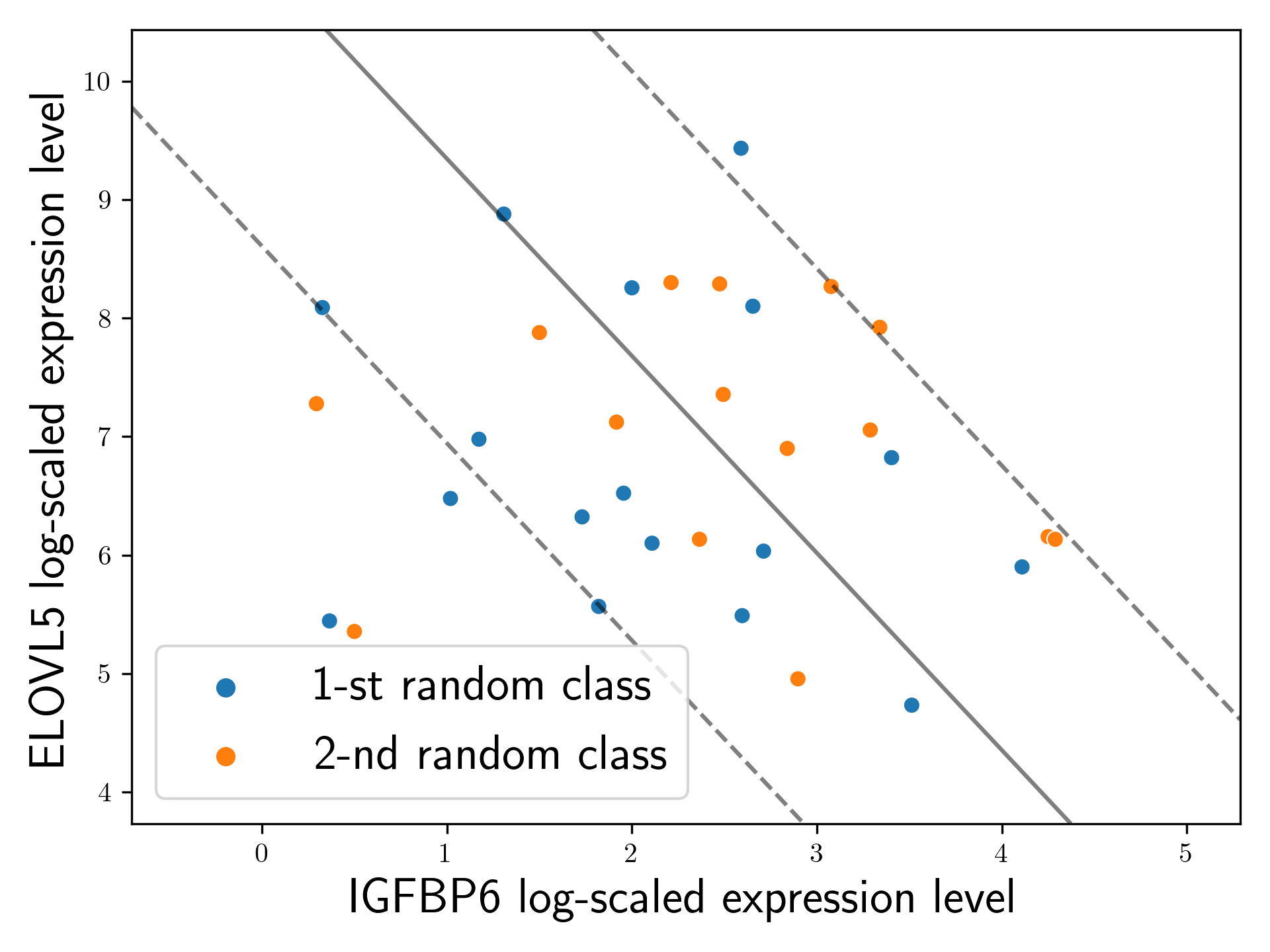}
    \end{subfigure}
    \centering
    \caption{
        \textbf{A}: The recurrent patients of ER-positive breast cancer differ significantly from non-recurrent ones, i.e., the constructed classifier is not ``random'', \emph{p}-value $\approx 0.002 < 0.05$.
        \textbf{B}: The classifier constructed on permuted labels is ``random'', since \emph{p}-value $\approx 0.76 > 0.05$ (i.e., the number of errors is too high to reject the homogeneity hypothesis $\mathrm{H}_0$).
    }
    \label{fig:igfbp6-elovl5}
\end{figure}

\section{Materials and methods}
\subsection{ER-positive breast cancer recurrence}
Microarray datasets used in~\cite{Galatenko2015} to train, filter and vaildate SVM-based linear classifiers are accessible via the corresponding GEO~\cite{Edgar2002} identifiers: GSE17705~\cite{Symmans2010}, GSE12093~\cite{Zhang2008}, GSE6532~\cite{Loi2008} and GSE3494~\cite{Miller2005}.
Sensitivity, specificity and other classification metrics for the constructed classifiers were obtained from Supplementary Table of~\cite{Galatenko2015}.

The TCGA-BRCA RNA-Seq dataset and its clinical annotation were downloaded from TCGA databank~\cite{Koboldt2012}.
Gene-level expression data for ER-positive patients with no recurrence within at least 7 years after surgery and ER-positive patients with death within 5 years after surgery were extracted from the corresponding metadata files.
RPKM (Reads Per Kilobase per Million mapped reads) value was used as an expression level characteristic~\cite{Mortazavi2008}.

Benjamini-Hochberg correction method~\cite{Benjamini1995} was applied to compute false discovery rates for \emph{p}-values.

\subsection{Implementation details}
The enumeration-based procedures for computing exact and approximate probabilities of linear separability with $m$ errors were implemented using parallelized approach in C++.
The SVM-based method was developed in Python using Sklearn~\cite{Pedregosa2011}.
The entire implementation has been released as an open-source Python package and is available at \url{https://github.com/zhiyanov/random-classifier}.

\section{Discussion}
Modern biology and medicine accumulate large datasets.
Among other applications, these datasets are used to develop various classifiers distinguishing
diseased vs healthy patients, or patients with high risk of recurrence 
by gene expression profiles.
A standard practice for validating such classifiers involves testing them on an independent dataset.
However, due to the inherent heterogeneity of biomedical data, classifiers achieving relatively low classification accuracy are still considered acceptable in practice.

We developed and characterized statistical tests answering the question whether a given classifier is ``random''.
Specifically, these tests determine if the classifier's performance is sufficiently high given the class distribution in the analyzed sample.

Our approach has two key components.
First, it leverages theoretical upper bounds on the probability that two samples drawn from the same distribution are near-linearly separable.
This analytical component enables immediate statistical significance bounds for linear classifiers in two-dimensional spaces.
Second, the test incorporates a computational procedure capable of estimating statistical significance for higher-dimensional samples.
While the computational component was implemented in the efficient C++ language with parallelization, its time complexity increased significantly with dimensionality.
To address this, we introduced an approximate computational mode that reduces processing time.
Using this approximation, we demonstrated that the theoretical upper bounds on probabilities were tight in cases of normally distributed data.

Finally, we applied the test to an important biomedical case study.
We evaluated a previously developed classifier that predicts risk of recurrence in patients with ER-positive breast cancer.
Our analysis demonstrated that the classifier was statistically significant, confirming the relevance of ELOVL5 and IGFBP6 genes in this type of cancer.

\section{Proofs}
\begin{proof}[Proof of Lemma~\ref{lem:error-bound}]
    Indeed, for a fixed $\SU$, all partitions of $X$ into $Y$ and $Z$ are equiprobable, occurring with probability $1 \, / \, \bin{n}{k}$.
    Consider a partition $p$ satisfying $\AU_{\leq m}$.
    Then, there exists a separating line $\ell$ such that the numbers of errors $m_Y$ and $m_Z$ both equal $h \leq m$.
    Otherwise, since $\max(m_Y, m_Z) = m$, we can shift $\ell$ parallelly (possibly, after a small rotation) until this condition holds.
    
    Since $k \geq 2 m - 1$, the line $\ell$ can be further transformed until the number of errors $m_Y$ and $m_Z$ both equal $m$.
    Indeed, there are exactly $k$ points of $\SU$ above $\ell$ and $l$ points below it.
    Let us describe a transformation step of $\ell$ that preserves this property.
    
    Rotate $\ell$ clockwise until it meets a point $x_1 \in \SU$.
    Continue rotating $\ell$ over $x_1$ until it encounters another point $x_2 \in \SU$.
    There are three cases to describe the number of points in the positive and negatives open half-planes generated by the line.
    In the first case, there are $k - 2$ points above $\ell$ and $l$ points below.
    Here, we shift $\ell$ down a little parallelly, resulting in $k$ and $l$ points on the positive and negative sides, respectively.
    In the second case, there are $k - 1$ points above $\ell$ and $l - 1$ below.
    Here, we perform a small rotation over the center of the segment generated by $x_1$ and $x_2$.
    The third case is symmetrical to the first, i.e., there are $k$ points above $\ell$ and $l - 2$ below.
    Here, we shift $\ell$ up a little parallelly.

    We perform this procedure until the line completes a $\pi$ turn.
    Since the resulting line satisfies the condition, the resulting numbers of errors $m_Y$ and $m_Z$ are equal.
    Moreover, during each step, they changed simultaneously by $+1$ or $-1$.
    However, the resulting errors of the negative half-plane include the correctly classified points of the positive half-plane of the initial line.
    Thus, $m_Z \geq k - h \geq m$, and by Discrete Intermediate Value Theorem~\cite{Johnsonbaugh1998}, there exists an intermediate line $\ell$ such that $m_Y$ and $m_Z$ both equal $m$, and $\AU_{\leq m} = \AU_m$.

    Note that the line $\ell$ corresponds to an error-free partition $p'$, dividing the points into two classes, i.e., those above $\ell$ and those below.
    Thus, we build a mapping from nearly separable partitions to partitions without errors.
    
    It remains to note that for any error-free partition $p'$, there are exactly $\bin{k}{m} \bin{l}{m}$ partitions with $m$ errors relative to $\ell$.
    Therefore,
    \[
        \PR[ \AU_{\leq m} \, ][ X = \SU ]
        \leq \bin{k}{m} \bin{l}{m} \, \PR[ \AU_0 \, ][ X = \SU ].
    \]
\end{proof}

\begin{proof}[Proof of Lemma~\ref{lem:accur-bound}]
    Let the event $\AU'_{\leq m}$ hold.
    Consider a separating line $\ell$ satisfying $\AU'_{\leq m}$.
    Then, $\ell$ can be shifted parallelly to have $m$ errors in terms of accuracy.
    Indeed, shifting $\ell$ up, at some point we reach a position where all points of $\SU$ are classified as $Z$.
    Here, we have the accuracy $k \geq m$, and by Discrete Intermediate Value Theorem~\cite{Johnsonbaugh1998}, such a line exists.
    Therefore, $\AU'_{\leq m } = \AU'_m$.

    Consider a partition $p$ with separating line $\ell$ and $m$ errors in terms of accuracy.
    Let $h$ be the number of points misclassified as $Y$.
    Removing them and the remaining $m - h$ misclassified points of class $Z$, we obtain an error-free partition $p'$.

    Let $N_m^{n, k}(\SU)$ denote the number of partitions satisfying $\AU'_m$.
    Then, 
    \begin{equation}
        N_m^{n, k}(\SU) \leq \sum_{h \leq m} \bin{n}{m} \bin{m}{h} \max_{\SU'} N_0^{n - m, k - h}(\SU'),
    \label{eq:accur-bound}
    \end{equation}
    since $\bin{n}{m} \bin{m}{h}$ is the number of ways to choose $h$ and $m - h$ misclassified points of classes $Y$ and $Z$, respectively.
    Note that $\bin{n}{k} \geq \bin{n - m}{k - h}$, since $k \leq n / 2$.
    Then, dividing both sides of~\eqref{eq:accur-bound} by $\bin{n}{k}$ and using the previous inequality, we get
    \[
        \PR[ \AU'_{\leq m} \, ][ X = \SU ] = \PR[ \AU'_{m} \, ][ X = \SU ]
        \leq \sum_{h \leq m} \bin{n}{m} \bin{m}{h} \, \PU_0^{n - m, k - h}.
    \]
\end{proof}

\begin{proof}[Proof of Theorem~\ref{thm:main-local}]
    We begin by recalling the definition of a $k$-set.
    \begin{definition}
        A subset $U$ of a planar set $\SU$ is called a $k$-set if $\abs*{U} = k$ and there exists a half-plane $G$ such that $G \cap \SU = U$.
    \label{def:k-set}
    \end{definition}

    Clearly, for a given $\SU$,
    \begin{equation}
        \PU_0 = N(\SU) \, / \, \bin{n}{k},
    \label{eq:triv}
    \end{equation}
    where $N(\SU)$ is the number of $k$-sets in $\SU$.
    The number of $k$-sets is well-studied.
    In this proof, we rely on results from~\cite{Erdos1973, Dey1998, Pach2006}.
    Following~\cite{Erdos1973, Dey1998}, we introduce the concept of \emph{directed $k$-graph}.
    
    Consider directed lines $\vec{vw}$ for every $v, w \in \SU$.
    Let $E_k$ be the set of all vectors $\vec{vw}$ such that exactly $k$ points of $\SU$ lie in the open half-plane corresponding to the positive side of the line $\vec{vw}$.
    We consider a \emph{directed $k$-graph} $G_k = (V_k, E_k)$ with the edges $E_k$ and vertices $V_k$ that are points of $\SU$ incident to the edges $E_k$.
    
    
    \begin{lemma}
        For any planar set $\SU$ of $n$ points in general position
        \[
            \PU_0 \leq \max\left( 6.3 \, n^{2 / 3} \sqrt[3]{k \, \abs*{V_{k - 1}}}, \, 103 \, n / 8 \right) \, / \, \bin{n}{k}.
        \] 
    \label{lem:proba-upper-vn}
    \end{lemma}
    \begin{proof}
        The proof is similar to the proof of~\cite[Theorem 3.3]{Dey1998}, with two modifications.
        The first change corresponds to~\cite[Lemma 3.2]{Dey1998}.
        We upper-bound the number of \emph{charged common tangents} (i.e., edge crossings) by $k \, \abs*{V_{k - 1}}$ instead of $k n$.
        
        The second change corresponds to~\cite[Theorem 3.3]{Dey1998}.
        Let $\crs(G)$ be the \emph{crossing number} (i.e., the lowest number of edge crossings) of the graph $G$ with $t$ edges and $n$ vertices.
        According to~\cite{Pach2006}, for $t \geq 103 \, n / 16$,
        \begin{equation}
            \crs(G) \geq \frac{ t^3 }{ 31.09 \, n^2 }.
        \label{eq:cross-lower}
        \end{equation}
        Thus, from~\cite[Theorem 3.3]{Dey1998} we have
        \[
            k \, \abs*{ V_{k - 1} } \geq \crs( G_{k - 1} ) \geq \frac{ t^3 }{ 31.09 \, n^2 },
        \]
        where $t = \abs*{ E_{k - 1} }$.
        Therefore, for $t \geq 103 \, n / 16$ we obtain 
        \[
            t \leq n^{ 2 / 3 } \sqrt[3]{ 31.09 \, k \, \abs*{ V_{k - 1} } } \leq 3.15 \, n^{ 2 / 3 } \sqrt[3]{ k \, \abs*{V_{k - 1}}}.
        \]
        Consequently, using the bounds from~\cite{Dey1998}, we get
        \[ 
            N(S) \leq \max\left( 6.3 \, n^{ 2 / 3} \sqrt[3]{k \, \abs*{V_{k - 1}}}, 103 \, n / 8 \right).
        \]

    \end{proof}
    
    Using Lemma~\ref{lem:proba-upper-vn} and the bound $\abs*{V_{k - 1}} \leq n$, we complete the proof.
\end{proof}

\begin{remark}
    The approach from~\cite[Theorem 2]{Montaron2005} can be used to improve the multiplicative constant in Theorem~\ref{thm:main-local} for specific $k$ and $n$.
\end{remark}

\begin{proof}[Proof of Theorem~\ref{thm:angle}]
    Recall that all partitions of $\SU$ of sizes $k$ and $l$ are equiprobable.
    Let $\DU^{\varphi}_{m}$ be the event that $\varphi$ is a near-separating directed angle with $m$ errors or less.
    Using Fubini's theorem, we obtain
    \[
        \ER \mu_{\leq m} = \ER \int_{0}^{2 \pi} \IR[ \DU^{\varphi}_{m} ] d\varphi = \int_{0}^{2 \pi} \PR[ \DU^{\varphi}_{m} ] d\varphi,
    \]
    where $\IR[ \DU^{\varphi}_{m} ]$ is the indicator function of $\DU^{\varphi}_{m}$.
    
    Using the same reasoning as in Lemma~\ref{lem:error-bound}, we get that the event $\DU^{\varphi}_{m}$ holds if and only if there exists a threshold $t \in \mathbb{R}$ such that the numbers of errors to the left and right of $t$ both equal to $h \leq m$.
    Thus, for $d = 1$, we get
    \[
        \PR[ \DU^{\varphi}_{m} ] = \frac{ \sum_{h \leq m} \bin{k}{h} \bin{l}{h} }{ \bin{n}{k} }.
    \]
    Therefore, by Markov's inequality,
    \[
        \PR[ \vphantom{ \mu'_{\leq m} } \mu_{\leq m} > x \, ][ X = \SU ] \leq \frac{ 2 \pi }{ x } \, \frac{ \sum_{h \leq m} \bin{k}{h} \bin{l}{h} }{ \bin{n}{k} }.
    \]

    Now, denote by $\DU'^{\varphi}_{m}$ the event that $\varphi$ is a near-separating directed angle with accuracy $m$ or less.
    Similarly to the proof of Lemma~\ref{lem:accur-bound}, one can show that $\DU'^{\varphi}_{m}$ holds if and only if the angle has the accuracy exactly $m$.
    There are $\bin{n}{m}$ options to choose $m$ errors from $\SU$.
    Moreover, for a fixed set of errors and $d = 1$, there is only one way to classify $h$ of them as $Z$ and other $m - h$ as $Y$, where $h \leq m$.
    Thus, we get
    \[
        \PR[ \DU'^{\varphi}_{m} ] \leq \frac{ \sum_{h \leq m} \bin{k}{h} \bin{l}{m - h} }{ \bin{n}{k} } = \frac{ \bin{n}{m} }{ \bin{n}{k} },
    \]
    and using Markov's inequality again, we prove the theorem.
\end{proof}

\begin{proof}[Proof of Theorem~\ref{thm:main-integral},~\ref{cond:proj-cond-prelim}]
    Let $A$ and $B$ be point sets of sizes $k$ and $l$, respectively.

    \begin{lemma}
        Suppose that $\SU = A \sqcup B$, and $A$, $B$ are linearly separable.
        Then, we can select a line $\ell$ passing through points $x \in A$ and $y \in B$ such that 
        \[
            \abs*{\ell^+ \cap A} = k - 1, \ 
            \abs*{\ell^- \cap B} = l - 1,
        \]
        where $\ell^+$ and $\ell^-$ are two open half-planes generated by $\ell$.
        We refer to $\ell$ as a \emph{proper breaking line}. 
    \label{lem:separ}
    \end{lemma}
    \begin{proof}[Proof of Lemma~\ref{lem:separ}]
        Consider a separating line $\ell$ such that the upper and lower half-planes contain exactly $k$ points of $A$ and $l$ points of $B$, respectively.
        Let us gradually shift $\ell$ downward until it passes through a point $x_1 \in \SU$.
        Without loss of generality suppose that $x_1 \in B$.
        Then, rotate it clockwise around $x_1$ until it encounters another point $x_2 \in \SU$.
        
        If $x_2 \in A$, then $\ell$ is a proper breaking line.
        Otherwise, continue rotating $\ell$ clockwise around $x_2$ until it encounters another point $x_3 \in \SU$, and so on.

        Suppose $x_i \not\in A$ for all $i \ge 1$.
        After at most $l$ steps, $\ell$ completes a full $2 \pi$ rotation.
        Since $A$ was always on the upper side of $\ell$ during the rotation, we reach a contradiction.
    \end{proof}

    For $Y = A$ and $Z = B$, denote by $\CU$ the event that $Y$ and $Z$ are separated by the proper breaking line passing through $Y_1$ and $Z_1$.
    Then, by Lemma~\ref{lem:separ},
    \[
        \PR[ \AU_0 ] = k l \, \PR[ \CU ].
    \]
    But
    \[
        \PR[ \CU ] = \ER\left[ \PR[ \CU \, ][ Y_1, Z_1 ] \right]
        = \ER[ \RU^{k - 1} (1 - \RU)^{l - 1} ],
    \]
    since $\RU$ is the probability that for fixed $X_2$ and $X_3$, the point $X_1$ belongs to the positive half-plain of the line through $X_2$ and $X_3$.
\end{proof}

\begin{proof}[Proof of Theorem~\ref{thm:main-integral},~\ref{cond:proj-cond}]
    By~\eqref{eq:triv},
    \begin{equation}
        \PR[ \AU_0 ] = \ER N(\SU) \, / \, \bin{n}{k}.
    \label{eq:basic-proba}
    \end{equation}
    But by~\cite[Theorem 1.3]{Leroux2023},
    \[
        \ER N(\SU) \leq \EU_2(k - 1, n) \leq 10 \, n \sqrt[4]{k},
    \]
    where $\EU_2(k - 1, n)$ denotes the expected number of $(k - 1)$-simplices.
\end{proof}

\begin{proof}[Proof of Theorem~\ref{thm:main-integral},~\ref{cond:spher-sym}]
    As~\cite[Theorem 1]{Barany1994} states, $\ER N(\SU)$ can be upper-bounded by $C n$.
    The constant $C$ can be explicitly computed from the proof.
    However, we found a typo in the estimation of $G(t + \Delta t) - G(t)$ at the end of~\cite[p. 247]{Barany1994}.
    The expression after the second equivalence should be 2 times greater, since there are 2 possible hyperplanes $\ell^+$ and $\ell^-$ in the definition of $F(\ell)$.
    Therefore, the resulting upper bound should be 2 times greater, i.e.,
    \[
        G(t + \Delta t) - G(t) \leq 2 \, d^2 \kappa_{d - 2} \left( \frac{ \kappa_{d - 2} }{ \kappa_{d - 1} } \right)^{d - 1} \Delta t.
    \]
    
    Using the proof of~\cite[Theorem 1]{Barany1994} with the described modification and~\eqref{eq:basic-proba}, we obtain
    \[
        \ER N(\SU) \leq E_2(k - 1, n)
        \leq 2 \, \frac{ 4 \, \kappa_0^2 }{ \kappa_1 (n - 1) } \frac{ \bin{n}{2} \bin{n - 2}{k - 1} }{ \bin{n - 2}{k - 1} }
        = \frac{ 8 \, n }{ \pi },
    \]
    where $\kappa_0 = 2$ and $\kappa_1 = 2 \pi$.
    Combining this inequality with~\eqref{eq:basic-proba}, we complete the proof.
\end{proof}

\begin{proof}[Proof of Theorem~\ref{thm:main-integral},~\ref{cond:normal}]
    Denote by $E$ the identity matrix.
    Without loss of generality, assume $X_i \sim \mathcal{N}(0, E)$, $i \leq n$, since affine transformations of points preserve separability.
    Let 
    \[
        \varphi(x) = \frac{1}{ \sqrt{2 \pi} } e^{-x^2 / 2}, \quad \Phi(x) = \int_{-\infty}^{x} \varphi(x) \, dx.
    \]

    The probability $\RU$ can be represented as
    \[
        \RU = \PR[
            X_{1, 2} - X_{1, 1} \frac{ X_{3, 2} - X_{2, 2} }{ X_{3, 1} - X_{2, 1} }
            > \frac{ X_{2, 2} X_{3, 1} - X_{2, 1} X_{3, 2} }{ X_{3, 1} - X_{2, 1} } \,
        ][
            X_2, X_3
        ],
    \]
    where $X_i = (X_{i, 1}, X_{i, 2})$ for $i \leq n$.
    For fixed $X_2$, $X_3$,
    \[
        X_{1, 2} - X_{1, 1} \frac{ X_{3, 2} - X_{2, 2} }{ X_{3, 1} - X_{2, 1} } \sim \mathcal{N}\left(0, 1 + \left( \frac{ X_{3, 2} - X_{2, 2} }{ X_{3, 1} - X_{2, 1} } \right)^2 \right).
    \]
    Thus,
    \[
        \RU = 1 - \Phi\left( \frac{ X_{2, 2} X_{3,1} - X_{2, 1} X_{3, 2} }{ \sqrt{ ( X_{3, 1} - X_{2, 1} )^2 + ( X_{3, 2} - X_{2, 2} )^2} } \right).
    \]

    Define the transformation
    \[
        U_1 = \frac{ X_{3, 1} - X_{2, 1} }{ \sqrt{2} }, \ V_1 = \frac{ X_{3, 2} - X_{2, 2} }{ \sqrt{2} }, \
        U_2 = \frac{ X_{3, 1} + X_{2, 1} }{ \sqrt{2} }, \ V_2 = \frac{ X_{3, 2} + X_{2, 2} }{ \sqrt{2} }.
    \]
    Then, 
    \[
        \RU = 1 - \Phi\left( \frac{W}{\sqrt{2}} \right), \quad W = \frac{ U_1 V_2 - V_1 U_2 }{ \sqrt{ U_1^2 + V_1^2 } } \sim \mathcal{N}(0, 1).
    \]
    Indeed, since $(U_1, V_1, U_2, V_2)$ are i.i.d. $\mathcal{N}(0, 1)$ r.v., it holds that
    \[
        \ER\left[ \ER[ e^{it W} \, ][ U_1, V_1] \right] 
        = \ER \exp\left( -\frac{ t^2 U_1^2 }{ 2 (U_1^2 + V_1^2) } - \frac{ t^2 V_1^2 }{ 2 (U_1^2 + V_1^2) } \right)
        = \exp(-t^2 / \, 2).
    \]
    Thus,
    \[
        \ER[ \RU^{k - 1} (1 - \RU)^{l - 1} ]
        = 2 \sqrt{\pi} \int_{\mathbb{R}} \Phi(x)^{l - 1} (1 - \Phi(x))^{k - 1} \varphi^2(x) \, dx.
    \]
    Applying the bound $\varphi(x) \leq 1 / \sqrt{2 \pi}$, we get
    \begin{align*}
        \ER[ \RU^{k - 1} (1 - \RU)^{l - 1} ]
        &\leq \sqrt{2} \, \int_{\mathbb{R}} \Phi(x)^{l - 1} (1 - \Phi(x))^{k - 1} \varphi(x) \, dx \\
        &= \sqrt{2} \, \int_{0}^{1} x^{l - 1} (1 - x)^{k - 1} \, dx
        = \sqrt{2} \, \frac{n}{ k l \, \bin{n}{k} }.
    \end{align*}
    Substituting this estimation into~\ref{cond:proj-cond-prelim}, we complete the proof.
\end{proof}

\section{Funding}
The research of APZ and AGT was supported by the framework of the Basic Research Program at HSE University.

\section{Conflict of interest}
All authors declare that they have no conflicts of interest.
Present affiliation of VVG: Evotec International GmbH, Germany.

\section{Acknowledgements}
The authors express their deep gratitude to Prof. S.V.~Shaposhnikov for useful comments and discussions.

\printbibliography

\end{document}